\documentclass{article}
\usepackage[utf8]{inputenc}
\usepackage[T1]{fontenc}

\usepackage{amsmath,amsfonts,bm}

\def\eqref#1{equation~\ref{#1}}

\def\1{\bm{1}}

\DeclareMathAlphabet{\mathsfit}{\encodingdefault}{\sfdefault}{m}{sl}
\SetMathAlphabet{\mathsfit}{bold}{\encodingdefault}{\sfdefault}{bx}{n}

\newcommand{\E}{\mathbb{E}}

\usepackage{amsmath,amssymb,amsthm,mathtools}
\usepackage{natbib}
\usepackage[ruled,vlined]{algorithm2e}
\usepackage{graphicx}
\usepackage{booktabs}
\usepackage{microtype}
\usepackage{hyperref}
\usepackage{url}
\usepackage{xcolor}
\hypersetup{colorlinks=true,linkcolor=black,citecolor=black,urlcolor=black}

\newtheorem{theorem}{Theorem}
\newtheorem{lemma}{Lemma}

\newtheorem{proposition}{Proposition}
\newtheorem{corollary}{Corollary}
\newtheorem{assumption}{Assumption}
\theoremstyle{remark}
\newtheorem{remark}{Remark}

\newcommand{\ind}[1]{\mathbb{I}\!\left\{#1\right\}}

\newcommand{\dist}{\mathrm{dist}}

\title{Multinomial Subset Routing with Sum-Max Rewards and Operational Constraints}
\author{Quan Zhou \and Yiyan Huang\thanks{Corresponding author: \texttt{huangyiyan@gbu.edu.cn}}}

\begin{document}
\maketitle

\begin{abstract}
We study online routing to subsets of experts under aggregate bandit feedback and long-run operational constraints. Expert contributions can be complementary: for each task dimension, the best selected expert determines the contribution, and the total reward aggregates these dimension-wise maxima. At the same time, capacity, budget, and fairness requirements impose lower and upper bounds on long-run expert activation frequencies. 
A fixed deterministic subset cannot generally satisfy such heterogeneous requirements, motivating a stochastic routing policy.
We formalize this problem as Multinomial Subset Routing (MSR). 
The learner maintains a distribution $q$ over $K$ experts, samples an expert independently $M$ times from $q$, and routes each task to the distinct sampled experts.
We propose OMD-Approachability, which combines online mirror descent with Blackwell's approachability to optimize MSR under two-sided operational constraints. We establish $O(T^{-1/2})$ average reward regret and expected constraint violation. We also quantify the approximation gap induced by a linear surrogate, and evaluate the approach on a real-world crowdsourcing dataset.
\end{abstract}

\section{Introduction}
\label{sec:intro}

We study online routing to small subsets of experts under aggregate feedback and long-run operational constraints. The selected experts can be complementary: different experts may be best suited to different aspects of a task, so the value of a team depends on the capabilities covered jointly. 
At the same time, deployment may impose long-term heterogeneous requirements on workload, cost, or fair participation. The learner therefore must optimize subset quality while controlling long-run expert usage from incomplete feedback.
These characteristics arise naturally in expert routing, crowdsourcing, and other resource-constrained allocation problems.

\paragraph{A motivating example.}
Consider a crowdsourcing platform that assigns each task to a small team of workers. A task may involve several skills, and different workers may be particularly effective on different skills. If $V_{k,i}\ge 0$ denotes the ability of worker $i\in[K]$ on skill $k\in[N]$, the quality of a selected subset $A\subseteq[K]$ is naturally represented by the sum-max reward
\[
    r(A)=\sum_{k=1}^{N}\max_{i\in A}V_{k,i}.
\]
The platform observes only the aggregate quality of the selected team and must simultaneously control long-run workload, cost, or participation. This gives a concrete instance of the online routing problem studied here.

We formalize the problem as \textbf{Multinomial Subset Routing (MSR)}. The learner maintains a distribution $q$ over $K$ experts, draws an expert independently $M$ times from $q$, and routes each task to the distinct sampled experts. Rather than choosing a deterministic subset directly, the learner learns the distribution that induces these subsets and their long-run activation frequencies.

Several lines of work address individual components of this problem. 
Mixture of Experts and LLM-routing methods learn input-dependent expert selection, often through learned gating or performance--cost trade-offs \citep{jacobs1991adaptive,shazeer2017outrageously,fedus2022switch,chen2020frugalml,shnitzer2023large,ong2024routellm}, and recent work studies online LLM routing with bandit feedback \citep{nguyen2024metallm,wei2025learning}. Sum-max and submodular bandits model complementary subset rewards \citep{pasteris2023sum,yue2011linear,yu2016linear,takemori2020submodular}, while constrained online learning and Bandits with Knapsacks control decisions coupled by long-run resource constraints \citep{badanidiyuru2018bandits,agrawal2016linear,slivkins2023contextual}. 
Our setting combines these ingredients: the learner chooses a routing distribution, observes only aggregate subset feedback, and must satisfy simultaneous lower and upper activation constraints.

This combination creates two coupled challenges. First, expert $i$ is activated with probability $1-(1-q_i)^M$, which is nonlinear in $q$. Second, the learner observes only the aggregate reward of the sampled subset, so the gradient of the expected sum-max reward must be estimated from subset-level feedback. The key mathematical insight is that the nonlinear activation constraints admit a linear surrogate: for every $q_i\in[0,1]$, $q_i\le 1-(1-q_i)^M\le Mq_i$, and the resulting coordinate-wise inequalities yield a distance-transfer guarantee from the surrogate target to the original activation constraints. For reward learning, the concavity of the expected sum-max reward \citep{pasteris2023sum} gives an importance-weighted gradient estimator based only on the aggregate subset reward.

We combine these ideas in \textbf{OMD-Approachability}. The primal update uses online mirror descent to maximize the concave reward while remaining in the supporting halfspace of the surrogate constraint set selected by the current dual direction. The dual update uses an OCO-based approachability procedure \citep{Hazan} to track long-run constraint pressure. Both updates operate on the same routing distribution, allowing reward optimization and operational control to proceed simultaneously.

Our contributions are fourfold. First, we formulate MSR, combining sum-max rewards, aggregate subset feedback, and simultaneous lower and upper long-run activation constraints. Second, we develop OMD-Approachability and establish $O(T^{-1/2})$ average reward regret and expected constraint violation under the benchmark-feasibility condition. Third, we quantify the approximation introduced by the linear surrogate, including a multiplicative $1-e^{-1}$ guarantee under upper constraints. Finally, we evaluate the method in simulated crowdsourcing environments constructed from real data, under both stationary and changing reward conditions. A contextual extension is also provided.

\section{Related Work}

\paragraph{Expert and model routing.}
Mixture of Experts methods route inputs to sparse subsets of expert subnetworks through learned gates \citep{jacobs1991adaptive,shazeer2017outrageously,fedus2022switch}. At the system level, routing among complete models has been studied through learned performance--cost trade-offs and model-response routers \citep{chen2022efficient,chen2020frugalml,shnitzer2023large,ong2024routellm,chen2024routerdc,feng2025graphrouter}, including recent online formulations based on multi-armed or contextual bandits \citep{nguyen2024metallm,wei2025learning}. These approaches typically select a model or expert subset directly, whereas MSR learns a distribution whose repeated multinomial sampling induces the subset and simultaneously controls long-run activation frequencies.

\paragraph{Constrained online learning.}
Bandits with knapsacks study sequential reward maximization under resource constraints \citep{badanidiyuru2018bandits}, with extensions to contextual, linear, combinatorial, and adversarial settings \citep{agrawal2016efficient,agrawal2016linear,slivkins2023contextual,sankararaman2018combinatorial,liu2022combinatorial,immorlica2022adversarial}. The concave-reward, convex-constraint formulation of \citet{agrawal2014bandits} is particularly related to our use of Blackwell approachability. In MSR, however, the relevant constraint is the activation probability $1-(1-q_i)^M$, which is nonlinear in the routing distribution. Our analysis therefore introduces a linear surrogate and then transfers the resulting guarantee back to the original activation constraints.

\paragraph{Sum-max, submodular, and combinatorial bandits.}
Sum-max rewards are monotone submodular and arise naturally in complementary subset selection. Existing submodular and combinatorial bandit methods study direct selection of deterministic super-arms under various feedback models \citep{yue2011linear,yu2016linear,takemori2020submodular,chen2017interactive,kveton2015tight,chen2016combinatorial,wang2018thompson}. K-max bandits consider the special case of a single maximum \citep{gopalan2013thompson,fourati2024combinatorial,chen2025continuous}. The closest reward model to ours is the sum-max bandit formulation of \citet{pasteris2023sum}, whose concavity and gradient characterization we use. Our setting differs in that the learner does not select a subset directly: it learns a routing distribution, observes only aggregate subset reward, and must satisfy whole-horizon two-sided operational constraints.

\section{Problem Setup}
\label{sec:setup}

We consider an online routing problem with $K$ available experts, $M$ routing draws per round, and horizon $T$. Let $[K]:=\{1,\dots,K\}$ and $\Delta_K:=\{q\in\mathbb{R}_+^K:\sum_iq_i=1\}$ denote the probability simplex. Each expert $i\in[K]$ has a \textit{known} cost vector $c_i\in[0,\bar c]^d$. At round $t$, nature specifies an \textit{unknown} reward function $r_t:2^{[K]}\to[0,1]$ with $r_t(\emptyset)=0$. The learner chooses $q_t\in\Delta_K$, draws $M$ experts $a_{t,1},\ldots,a_{t,M}$ independently from $q_t$, and routes to the distinct sampled set $A_t:=\{a_{t,1},\ldots,a_{t,M}\}$. It observes only the aggregate reward $r_t(A_t)$ and incurs the realized cost
\[
    \ell_t:=\sum_{i=1}^K c_i\,\mathbf 1\{i\in A_t\}.
\]
The reward sequence is fixed before the interaction and is unknown to the learner.

\paragraph{Sum-max reward.}
We consider the sum-max reward class introduced by~\citet{pasteris2023sum}. Fix $N\in\mathbb{N}$. At each round $t$, there is a matrix $V_t\in\mathbb{R}^{N\times K}$ such that
$$
r_t(A)
=
\sum_{k=1}^N \max_{i\in A} V_{t,k,i},
\qquad \forall A\subseteq[K],
\qquad r_t(\emptyset)=0.
$$

For $q\in\Delta_K$, let $B(q)$ be the \textit{distinct} set obtained from $M$ independent draws from $q$, and define the expected reward
\[
    \Phi^{r_t}(q):=\mathbb E[r_t(B(q))].
\]

\paragraph{Operational constraints.}
The learner must also satisfy long-run operational constraints. Let $S^{\mathrm{ori}}=\prod_{j=1}^d[l_j,u_j]$ denote the axis-aligned target set for the time-averaged realized cost. Under $q$, action $i$ is activated with probability $1-(1-q_i)^M$, so the corresponding expected cost is nonlinear in $q$. We therefore use the linear surrogate $\sum_i q_i c_i$ and the compact target $S=\prod_{j=1}^d S^{(j)}$, where
\[
S^{(j)}=[0,u_j/M]\ \text{for an upper-bound coordinate},\qquad
S^{(j)}=[l_j,\bar c]\ \text{for a lower-bound coordinate}.
\]
Section~\ref{sec:reduction} establishes the distance-transfer guarantee connecting the surrogate and original constraints.

\paragraph{Benchmark and performance criteria.}
We compare the learner with the best fixed routing distribution in hindsight that is feasible for the surrogate target:
\begin{equation}
  \mathrm{OPT}_{\mathrm{sur}} := \max_{q \in \Delta_K}\; \tfrac{1}{T}\sum_{t=1}^T \Phi^{r_t}(q)
  \quad \text{s.t.} \quad
  \sum_{i=1}^K q_i c_i \in S.
\label{equ:opt-define}
\end{equation}
We refer to $\mathrm{OPT}_{\mathrm{sur}}$ as the \emph{surrogate-feasible benchmark}. For $u\in\mathbb R^d$ and a nonempty closed set $A$, let $\mathrm{dist}(u,A):=\min_{x\in A}\|x-u\|$. We evaluate the average objective regret and realized constraint regret by
\begin{align}
\mathcal{R}_1(T):= \mathrm{OPT}_{\mathrm{sur}}-\tfrac1T\sum_{t=1}^T\Phi^{r_t}(q_t),\qquad
\mathcal{R}_2(T):=\mathrm{dist}\!\left(\tfrac1T\sum_{t=1}^T\ell_t,\;S^{\mathrm{ori}}\right).
\end{align}
Thus $\mathcal R_1(T)$ measures the reward gap to the best fixed feasible routing distribution, while $\mathcal R_2(T)$ measures violation of the original operational constraints using realized activations.

\section{OMD-Approachability}
\label{sec:algorithm}

We now develop OMD-Approachability. The method combines two structural facts about MSR. The expected sum-max reward is concave in the routing distribution and admits an unbiased importance-weighted gradient estimator, while the nonlinear expected activation cost can be controlled through the linear surrogate introduced above. These two ingredients allow reward optimization and constraint control to be performed simultaneously through a common routing distribution.

\subsection{Reward and Constraint Reduction}
\label{sec:reduction}

The first ingredient is the concavity and gradient structure of the expected sum-max reward.

\begin{lemma}[Concavity and gradient of $\Phi^r$]
\label{lem:phi-concave-grad}
Let $r$ be a sum-max function. Then $\Phi^r$ is concave on $\Delta_K$. For every $q\in\Delta_K$,
\[
\partial_i \Phi^r(q)=\mathbb E\!\left[\frac{r(B(q))}{q_i}\sum_{j=1}^M\mathbf 1\{b_j(q)=i\}\right].
\]
\end{lemma}
The concavity statement is Theorem~2.3 of~\citet{pasteris2023sum}, and the gradient expression follows from their Lemma~6.7. This identity motivates the bandit gradient estimator in Algorithm~\ref{alg:approachability-OMD}.

For the constraints, under routing distribution $q_t$, action $i$ is activated with probability $1-(1-q_{t,i})^M$. Hence
\begin{equation}
  \mathbb E[\ell_t]=\sum_{i=1}^K\bigl(1-(1-q_{t,i})^M\bigr)c_i.
  \label{eq:exact-cost}
\end{equation}
Since $p\le 1-(1-p)^M\le Mp$ for $p\in[0,1]$, every coordinate $j\in[d]$ satisfies
\begin{equation}
  \left(\sum_{i=1}^K q_{t,i}c_i\right)^{(j)}
  \;\le\;\mathbb E[\ell_t]^{(j)}
  \;\le\;M\left(\sum_{i=1}^K q_{t,i}c_i\right)^{(j)}.
  \label{eq:sandwich}
\end{equation}
The rescaling in the definition of $S$ is chosen exactly from these inequalities: upper bounds are divided by $M$, whereas lower bounds are unchanged.

\subsection{OMD-Approachability Algorithm}

We next combine online mirror descent for reward learning with Blackwell approachability for constraint control. Algorithm~\ref{alg:approachability-OMD} uses the support function $h_S(w):=\max_{x\in S}w^\top x$. Because $S$ is a compact axis-aligned box, $h_S$ is finite and decomposes coordinate-wise:
 {
\[
  h_{[0,\,u_j/M]}(w_j) =
  \begin{cases}
    (u_j/M)\, w_j & w_j \ge 0,\\[2pt]
    0 & w_j < 0,
  \end{cases}
  \qquad
  h_{[l_j,\,\bar{c}]}(w_j) =
  \begin{cases}
    \bar{c}\, w_j & w_j \ge 0,\\[2pt]
    l_j w_j & w_j < 0.
  \end{cases}
\]
}
By~\cite{pasteris2023sum}, $g_{t,i}$ is an unbiased estimator of the $i$th coordinate of $\nabla\Phi^{r_t}(q_t)$. Its second moment grows when $q_{t,i}$ approaches zero, so we impose the following lower-bound condition on the benchmark solution.

\begin{assumption}\label{ass:non-zero}
The benchmark problem \eqref{equ:opt-define} is feasible and admits an optimal solution $q\in\Delta_K$ satisfying $q_i\ge \gamma$ for all $i\in[K]$.
\end{assumption}

\begin{remark}[feasibility of the QP]\label{rem:feasible}
Assumption~\ref{ass:non-zero} also guarantees feasibility of the quadratic program~\eqref{equ:online-opt}. Let $q^\star$ be an optimal solution of~\eqref{equ:opt-define}. Its feasibility gives $\sum_{i=1}^Kq_i^\star c_i\in S$. Lemma~\ref{lem:distance2support} then implies
\[
w^\top \sum_{i=1}^K q_i^\star c_i - h_S(w) \le 0,
\qquad \forall\, \|w\|\le 1.
\]
The algorithm keeps $w_t$ in the unit Euclidean ball, so the inequality holds at $w=w_t$. Assumption~\ref{ass:non-zero} also gives $q_i^\star\ge\gamma$ for every $i$. Thus $q^\star$ satisfies every constraint in~\eqref{equ:online-opt}, proving that the quadratic program is feasible at each round.
\end{remark}

At each round, the dual OCO subroutine selects a direction $w_t$ that summarizes the current constraint pressure. The primal QP then updates the routing distribution using the previous reward-gradient estimate while enforcing the supporting halfspace of $S$ induced by $w_t$. After sampling the subset and observing its aggregate reward, the learner forms the importance-weighted gradient estimate used in the next primal update. The reward and constraint updates therefore share the same routing distribution and use separate feedback quantities.

\begin{algorithm}[H]\label{alg:approachability-OMD}
\caption{Online Mirror Descent with Approachability Constraints}
\begin{enumerate}
\item \textbf{Input:} rescaled set $S$, OCO algorithm $\mathcal{A}$, cost vectors $c_i\in[0,\bar{c}]^d$ for $i\in[K]$, floor $\gamma$ and learning rate $\eta$.

\item Set $\mathbb{B} \subset \mathbb{R}^d$ to be the unit Euclidean ball, as decision set for $\mathcal{A}$.
\item Set $q_{0,i}:=1/K$ and initialize $g_{0,i}$ arbitrarily for $i\in[K]$.
\item For $t = 1, \ldots, T$:
\begin{enumerate}
    \item Set $f_{t-1}(w) := w^\top \left(\sum_{i=1}^K q_{t-1,i} c_i\right) - h_S(w)$.
    \item Query $\mathcal{A}$: $w_t \leftarrow \mathcal{A}(f_1, \ldots, f_{t-1})$
    \item Set $q_t$ as the solution of: 
    \begin{equation}
    \label{equ:online-opt}
    \begin{split}
        \max_{q\in\Delta_K:\; q\ge\gamma} & \quad q^\top g_{t-1} -\tfrac{\eta}{2}\|q-q_{t-1}\|^2 \\
        \text{s.t. }& w_t^\top  \sum_{i=1}^K q_i\; c_i - h_S(w_t) \leq 0
    \end{split}
    \end{equation}
\item Draw $a_{t,1},\dots,a_{t,M}$ independently from $q_t$\;
\item Set $A_t \leftarrow \{a_{t,j} : j \in [M]\}$\;
\item Observe $r_t(A_t)$\;
\item For each $i \in [K]$, define
\[
g_{t,i} \leftarrow \tfrac{r_t(A_t)}{q_{t,i}} \sum_{j=1}^M \ind{a_{t,j}=i}
\]
\end{enumerate}
\end{enumerate}
\end{algorithm}

\begin{lemma}[Distance transfer]
\label{lem:dist-transfer}
For any sequence $q_1,\ldots,q_T\in\Delta_K$,
\[
\operatorname{dist}\!\left(
\frac1T\sum_{t=1}^T\mathbb E[\ell_t],\,S^{\mathrm{ori}}
\right)
\le
M\,\operatorname{dist}\!\left(
\frac1T\sum_{t=1}^T\sum_{i=1}^K q_{t,i}c_i,\,S
\right).
\]
\end{lemma}
The proof is coordinate-wise and is given in Appendix~\ref{app:dist-transfer}. In particular, if the average surrogate cost lies in $S$, then the average expected activation cost lies in $S^{\mathrm{ori}}$. This result is the link between the approachability update below and the original operational constraints.

\section{Theoretical Guarantees}
\label{sec:theory}

The following results characterize the two performance criteria defined in Section~\ref{sec:setup}. The reward bound controls the gap to the surrogate-feasible benchmark, while the constraint bound transfers approachability of the surrogate set to the original realized-cost constraints.

\subsection{Reward and constraint guarantees}

We first state the reward and constraint guarantees. Complete proofs are given in the appendix.

\begin{theorem}[Reward regret]\label{thm:reg1}
Let $q_t$ for $t\in[T]$ be the output of Algorithm~\ref{alg:approachability-OMD} with proximal parameter $\eta^* = \sqrt{2TKM\bigl(\tfrac{1}{\gamma}+(M-1)\bigr)}$. Then
\[
T\;\mathbb{E}[\mathcal{R}_1(T)]\le 2\sqrt{2TKM\!\left(\tfrac{1}{\gamma} + (M-1)\right)}.
\]
\end{theorem}

The constraint guarantee follows from the approachability update and the concentration of the realized activation costs.

\begin{theorem}[Constraint regret]\label{thm:reg2}
Let $\mathcal{A}$ be an OCO algorithm with regret $\mathrm{Regret}_T(\mathcal{A})$. Then for an absolute constant $L>0$,
\[
\mathbb{E}\left[\mathcal{R}_2(T)\right]\le M\cdot\tfrac{\mathrm{Regret}_T(\mathcal{A})}{T} +   {M}\bar{c}L\sqrt{\tfrac{d\log(2d)}{T}}.
\]
\end{theorem}

With online gradient descent (OGD) as $\mathcal A$, $\mathrm{Regret}_T(\mathcal A)=O(\sqrt T)$. Theorem~\ref{thm:reg1} and Theorem~\ref{thm:reg2} therefore establish $O(T^{-1/2})$ average reward regret and expected constraint violation.

\subsection{Benchmark approximation from the Linear Surrogate}
\label{subsec:surrogate-suboptimality}

The main regret guarantee is relative to the surrogate-feasible benchmark $\mathrm{OPT}_{\mathrm{sur}}$. We next quantify the price of replacing the original nonlinear activation constraint by its linear surrogate.

Our algorithm enforces the linear surrogate constraint
$\sum_i q_i c_i\in S$, rather than the original expected activation constraint.
To quantify the resulting benchmark gap, define the best fixed routing policy
under the latter constraint by
\begin{equation}
\mathrm{OPT}_{\mathrm{ori}}
:= \max_{q \in \Delta_K}\; \tfrac{1}{T}\sum_{t=1}^T \Phi^{r_t}(q)
\quad \text{s.t.} \quad
\sum_{i=1}^K \bigl[1-(1-q_i)^M\bigr]c_i \in S^{\mathrm{ori}}.
\label{eq:subopt-gap}
\end{equation}
Since $q_i\le 1-(1-q_i)^M\le Mq_i$, surrogate feasibility implies original
expected-cost feasibility, and hence $0\le \mathrm{OPT}_{\mathrm{ori}}-\mathrm{OPT}_{\mathrm{sur}}\le1$.

\paragraph{Two-sided constraints.}
For a vector $x$, let $x^{(j)}$ denote its $j$-th coordinate.
Let $\mathcal U$ and $\mathcal L$ index the upper and lower operational
constraints. Suppose there exist $q^{\mathrm{ref}}\in\Delta_K$ and $\sigma>0$ such that
\begin{equation}
M\Bigl(\sum_i q_i^{\mathrm{ref}}c_i\Bigr)^{(j)}\le u_j-\sigma\quad(j\in\mathcal U),
\qquad
\Bigl(\sum_i q_i^{\mathrm{ref}}c_i\Bigr)^{(j)}\ge l_j+\sigma\quad(j\in\mathcal L).
\label{eq:subopt-margin}
\end{equation}
For $q\in\Delta_K$, define the surrogate constraint violation
\begin{equation}
\delta(q):=
\max\left\{
\max_{j\in\mathcal U}\left[M\Bigl(\sum_i q_i c_i\Bigr)^{(j)}-u_j\right]_+,
\max_{j\in\mathcal L}\left[l_j-\Bigl(\sum_i q_i c_i\Bigr)^{(j)}\right]_+
\right\},
\label{eq:subopt-residual}
\end{equation}
where $[x]_+:=\max\{x,0\}$. Thus, $\delta(q)$ is the maximum violation of the surrogate upper and lower coordinate constraints.

\begin{proposition}
\label{prop:subopt-margin}
Under~\eqref{eq:subopt-margin}, let $q^{\mathrm{ori}}$ attain
$\mathrm{OPT}_{\mathrm{ori}}$. Then
\begin{equation}
0\le \mathrm{OPT}_{\mathrm{ori}}-\mathrm{OPT}_{\mathrm{sur}}
\le
\frac{\delta(q^{\mathrm{ori}})}{\sigma+\delta(q^{\mathrm{ori}})}
\left(
\mathrm{OPT}_{\mathrm{ori}}-\frac1T\sum_{t=1}^T\Phi^{r_t}(q^{\mathrm{ref}})
\right)
\le
\frac{\delta(q^{\mathrm{ori}})}{\sigma+\delta(q^{\mathrm{ori}})}.
\label{eq:subopt-margin-bound}
\end{equation}
\end{proposition}

Indeed, with $\theta=\delta(q^{\mathrm{ori}})/(\sigma+\delta(q^{\mathrm{ori}}))$ and
$\widetilde q=(1-\theta)q^{\mathrm{ori}}+\theta q^{\mathrm{ref}}$, the distribution
$\widetilde q$ is surrogate-feasible; concavity of
$\frac1T\sum_t\Phi^{r_t}$ gives~\eqref{eq:subopt-margin-bound}.

\paragraph{Upper constraints only.}
When only upper constraints are present and a dull action is available,
the sum-max structure gives a stronger, multiplicative guarantee:
\begin{corollary}
\label{cor:subopt-null}
Suppose all constraints are upper bounds and there is an action $i_0$ with
$c_{i_0}=0$ and $r_t(A\cup\{i_0\})=r_t(A)$ for all $A\subseteq[K]$ and $t$.
Then, for every $M\ge2$,
\begin{equation}
\frac{\mathrm{OPT}_{\mathrm{sur}}}{\mathrm{OPT}_{\mathrm{ori}}}
\ge 1-\left(1-\frac1M\right)^M\ge1-e^{-1}.
\label{eq:subopt-null-bound}
\end{equation}
The factor $1-e^{-1}$ is the classical approximation factor that also appears monotone submodular maximization \citep{calinescu2011maximizing}; here it arises from
converting activation marginals into an iid multinomial sampling distribution.
The factor is sharp; see Appendix~\ref{app:subopt-null-proof}.
\end{corollary}

Together, these results separate the overall performance loss into two components: the $O(T^{-1/2})$ online optimization error relative to $\mathrm{OPT}_{\mathrm{sur}}$, and the approximation gap $\mathrm{OPT}_{\mathrm{ori}}-\mathrm{OPT}_{\mathrm{sur}}$ induced by the surrogate constraint. The latter can be controlled under a strict-feasibility margin and admits the sharper multiplicative guarantee above in the upper-bound case.

\subsection{Contextual extension}
\label{sec:contextual}

We next allow the routing policy to depend on an observed context. At round $t$, the learner observes $x_t\in\mathcal X$, where $\mathcal X$ is finite, before selecting the routing distribution. The context sequence can be arbitrary. The learner maintains a separate primal OMD state for each context and uses one shared dual iterate $w_t$ so that the operational constraint is enforced on the aggregate sequence. The contextual benchmark chooses the best fixed feasible routing policy separately for each context:
\[
  \mathrm{OPT}^{\mathrm{ctx}} := \tfrac{1}{T}\sum_{x\in\mathcal{X}} \;\max_{\substack{q\in\Delta_K:\,q_i\ge\gamma \\ \sum_i q_i c_i \in S}} \;\sum_{t:\,x_t=x} \Phi^{r_t}(q).
\]
The contextual objective regret is $\mathcal{R}_1^{\mathrm{ctx}}(T):=\mathrm{OPT}^{\mathrm{ctx}}-\tfrac1T\sum_{t=1}^T\Phi^{r_t}(q_t^{(x_t)})$. Appendix~\ref{app:contextual} gives the complete algorithm.

\begin{corollary}[Contextual reward regret]\label{cor:contextual-reg1}
Run the contextual algorithm (Appendix~\ref{app:contextual}) with per-context anytime proximal parameter $\eta_t = a\sqrt{n_{x_t}}$, where $a=2\sqrt{KM(\tfrac1\gamma+M-1)}$ and $n_{x_t}$ is the visit count for context $x_t$ up to round $t$. Then, without prior knowledge of $T$ or $\{T^{(x)}\}$,
\[
  T\,\mathbb{E}\!\left[\mathcal{R}_1^{\mathrm{ctx}}(T)\right] \;\le\; 4\sqrt{\lvert\mathcal{X}\rvert\,T K M\!\left(\tfrac{1}{\gamma}+(M-1)\right)}.
\]
Relative to the non-contextual result, the bound incurs a factor of $O(\sqrt{|\mathcal X|})$ from learning separate primal policies across contexts.
\end{corollary}

\begin{remark}[Shared dual, decoupled primal]
Algorithm~\ref{alg:contextual-OMD} maintains one primal iterate $q^{(x)}$ for each $x\in\mathcal X$ and a single shared dual iterate $w_t$. Only $q^{(x_t)}$ is updated at round $t$, while the dual update uses the played policy $q_t=q^{(x_t)}$. The reward analysis therefore separates across contexts. Constraint control depends on the aggregate surrogate cost
\(
\frac1T\sum_{t=1}^T\sum_{i=1}^Kq_{t,i}c_i
\),
which is coupled through the shared dual process. The proof of Theorem~\ref{thm:reg2} uses only this aggregate sequence and the shared dual updates, so the same constraint guarantee applies to Algorithm~\ref{alg:contextual-OMD}.
\end{remark}

\section{Experiments}
\label{sec:experiments}

We construct a simulated online routing environment from the TREC 2010 Crowdsourcing Track data. The experiments are organized around the two claims established above: convergence of the objective and constraint regret, and the reward--feasibility trade-off induced by operational constraints.

\subsection{Experimental Setup}
\label{sec:exp-data}

We use the TREC crowdsourcing data distributed with the SpectralMethodsMeetEM repository of~\citet{zhang2014spectral}. For each random seed, we sample $K=50$ workers, route each task to $M=5$ workers, and construct the sum-max reward from worker reliability across task groups. The two operational coordinates measure normalized workload and activation of the highest-accuracy workers. We evaluate both a stationary stochastic environment and a regime-switching environment in which worker reliability changes over time. Unless stated otherwise, the horizon is $T=10{,}000$. Dataset preprocessing, construction of $V^{(t)}$, exact constraint thresholds, and all hyperparameters are reported in Appendix~\ref{app:experiments}.

We compare OMD-Approachability with Random and Greedy routing, Combinatorial UCB (CUCB)~\citep{chen2013combinatorial}, Combinatorial Bandits with Knapsacks (CBwK)~\citep{sankararaman2018combinatorial,liu2022combinatorial}, Thompson Sampling (TS)~\citep{wang2018thompson}, a constrained TS variant (TS-BwK)~\citep{deb2025thompson}, and sparse Mixture-of-Experts gating (MoE)~\citep{shazeer2017outrageously}. Because the dataset provides only aggregate team reward in our online simulation, all baselines are evaluated under the same feedback available to OMD-Approachability. Full implementation details are deferred to Appendix~\ref{app:experiments}.

\subsection{Regret Convergence}
\label{sec:exp-sweep}

We first evaluate empirical counterparts of $\mathcal R_1(T)$ and $\mathcal R_2(T)$ under $N\in\{3,5\}$ reward dimensions and tight, medium, and loose operational constraints. OMD-Approachability updates its routing distribution every $B=10$ rounds, so Figure~\ref{fig:sweep} indexes the curves by the number of OMD updates, $n=t/B$, and includes an $O(n^{-1/2})$ reference. The objective quantity uses the final learned policy as an empirical proxy for the unavailable hindsight-optimal feasible policy; Appendix~\ref{app:experiments} gives the exact Monte Carlo construction.

Across both values of $N$, the empirical objective regret decreases toward zero for all three constraint levels after a short initial transient. The constraint regret separates the settings more clearly. Under the tight constraint, the running-average realized cost initially violates the target and then decreases to approximately zero as the approachability update adjusts the routing distribution. The medium configuration exhibits only a short initial violation, while the loose configuration remains feasible throughout. The same qualitative behavior is observed for $N=3$ and $N=5$, with slower constraint adjustment in the higher-dimensional reward instance.

\begin{figure}[h]
\centering
\includegraphics[width=0.9\linewidth]{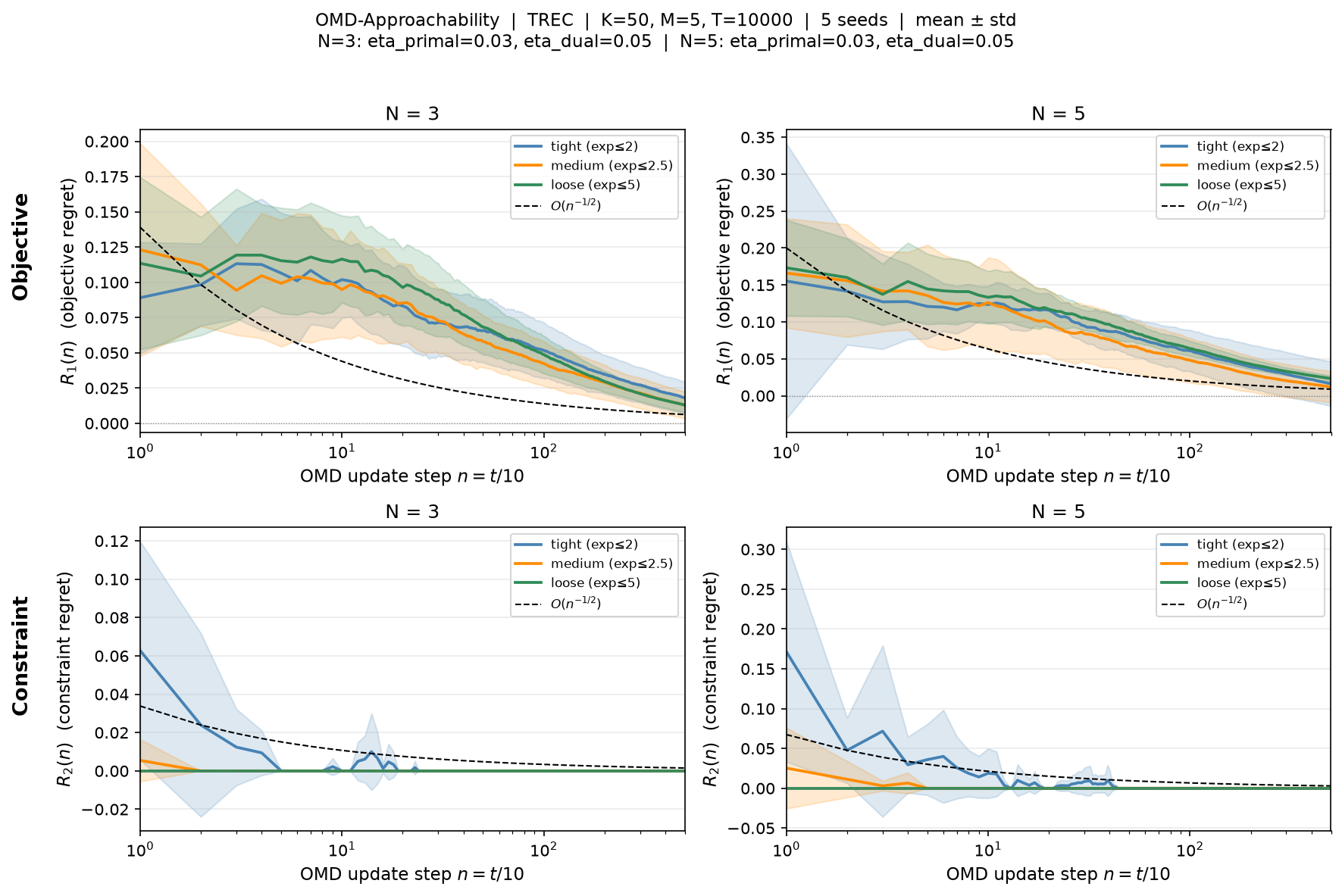}
\caption{Empirical objective regret (top) and realized constraint regret (bottom) versus OMD update step $n=t/10$, for $N=3$ (left) and $N=5$ (right). Curves show the mean $\pm$ one standard deviation over $5$ seeds; dashed curves show an $O(n^{-1/2})$ reference.}
\label{fig:sweep}
\end{figure}

\subsection{Reward and Feasibility under Constraint Tightness}
\label{sec:exp-crowd}

We next compare the eight policies across tight, medium, loose, and unconstrained settings, using $N=3$ and $20$ independent seeds in both stationary and regime-switching environments. Figure~\ref{fig:crowd-sweep} reports time-averaged reward together with the fraction of feasible runs, while Figure~\ref{fig:crowd-reward} shows rolling reward under the tight and loose constraints. The exact thresholds defining the four constraint levels are given in Appendix~\ref{app:experiments}.

Under the tight constraint, OMD-Approachability is feasible in every seed in both environments while retaining reward close to the unconstrained methods. TS-BwK is feasible in most, but not all, runs, and CBwK is substantially less reliable, especially when worker reliability switches over time. CUCB, TS, and Greedy obtain high reward but are rarely feasible because they do not control cumulative expert usage. As the constraint is relaxed, feasibility differences shrink and the competitive methods approach similar reward levels. The rolling curves show the corresponding adaptation: OMD-Approachability initially sacrifices reward while adjusting expert usage, then approaches the reward range of the unconstrained methods while maintaining feasibility. This pattern persists under regime switching, where the constraint-aware baselines exhibit larger feasibility losses.

\begin{figure}[h]
\centering
\includegraphics[width=\linewidth]{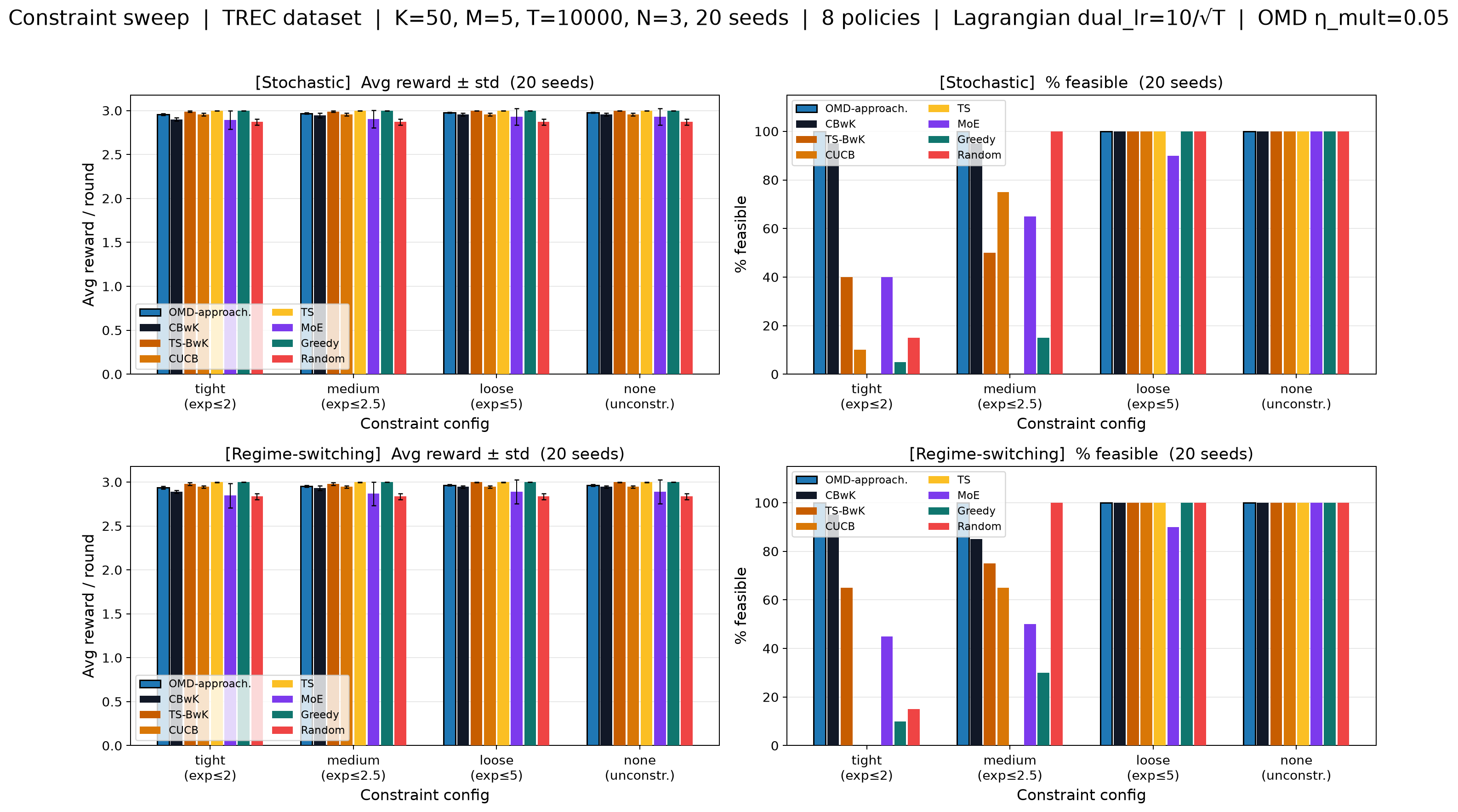}
\caption{Constraint-tightness sweep with $K=50$, $M=5$, $N=3$, $T=10{,}000$, and $20$ seeds. The left column shows mean $\pm$ standard deviation of time-averaged reward, and the right column shows the percentage of feasible seeds. Rows correspond to the stationary and regime-switching environments.}
\label{fig:crowd-sweep}
\end{figure}

\begin{figure}[ht]
\centering
\includegraphics[width=\linewidth]{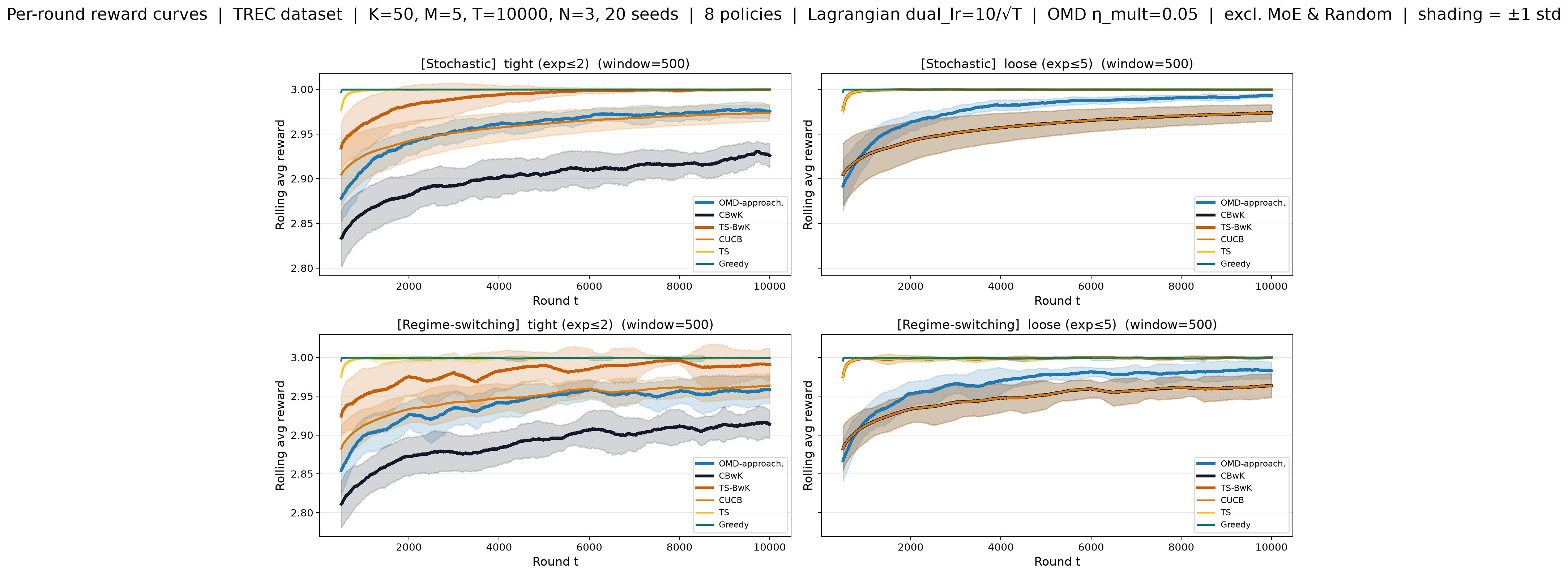}
\caption{Rolling-average reward (window $5{,}000$) under the tight (left) and loose (right) constraints across $20$ seeds. Rows correspond to the stationary and regime-switching environments. MoE and Random are omitted because their higher per-round variance obscures the comparison among the remaining methods.}
\label{fig:crowd-reward}
\end{figure}

\section{Conclusion}

We introduced Multinomial Subset Routing for online subset routing with sum-max rewards, aggregate bandit feedback, and long-run operational constraints. The main difficulty is that repeated multinomial sampling induces nonlinear activation probabilities, while the learner observes only aggregate subset rewards. OMD-Approachability addresses these two challenges by combining online mirror descent for the concave reward with Blackwell approachability for a linear surrogate of the operational constraints. The resulting guarantees give $O(T^{-1/2})$ average objective regret and expected constraint regret relative to the surrogate-feasible benchmark, together with explicit control of the approximation gap to the original benchmark.
A contextual extension is also provided, while future work includes richer contextual routing models and alternative subset-generation mechanisms that may provide tighter control of the original nonlinear activation constraints.

\subsection*{AI Use Statement}

In this work, we used generative AI tools to assist with experimental
implementation, including writing and debugging code. We also used
generative AI tools to assist with drafting and editing the manuscript,
including parts describing the theoretical development, and to identify
relevant literature. The research direction, problem formulation, and
main theoretical contributions were developed by the authors. We
reviewed and verified all AI-assisted text, code, citations, and
experimental results, and take responsibility for the final content of this work.

\bibliographystyle{plainnat} 
\bibliography{ref} 

\appendix

\section*{Appendix Contents}
\begin{itemize}
  \item Appendix~\ref{app:experiments}: Experimental Details \dotfill \pageref{app:experiments}
  \item Appendix~\ref{app:auxiliary}: Auxiliary Lemmas \dotfill \pageref{app:auxiliary}
  \item Appendix~\ref{app:thm1}: Proof of Theorem~\ref{thm:reg1} \dotfill \pageref{app:thm1}
  \item Appendix~\ref{app:thm2}: Proof of Theorem~\ref{thm:reg2} \dotfill \pageref{app:thm2}
  \item Appendix~\ref{app:contextual}: Contextual Algorithm and Proof of Corollary~\ref{cor:contextual-reg1} \dotfill \pageref{app:contextual}
\end{itemize}

\section{Experimental Details}
\label{app:experiments}

\subsection{Dataset and construction}
We use the TREC 2010 Crowdsourcing Track data distributed with the SpectralMethodsMeetEM repository of~\citet{zhang2014spectral}. The file \texttt{trec\_crowd.txt} contains each worker's binary response to a task, and \texttt{trec\_truth.txt} contains the corresponding gold label. Joining the files by task ID and retaining tasks with a gold label yields $677$ workers, $2{,}275$ tasks, and $12{,}863$ annotations. For each annotation, we define \emph{correct}$=\mathbf 1\{\text{response}=\text{gold label}\}$ and compute each worker's overall accuracy. Removing workers with fewer than $5$ annotations leaves $270$ workers.

For each seed, we sample $K=50$ workers uniformly without replacement. All skill and cost statistics are recomputed within this pool. We partition the $2{,}275$ tasks into $N$ groups by sorted task ID and set $V_{\mathrm{base}}[n,i]$ to worker $i$'s accuracy in group $n$, using the worker's overall accuracy when that group has no annotation. Each worker has two cost coordinates. The \emph{workload} coordinate is its annotation count divided by the maximum count in the sampled pool. The \emph{expert} coordinate equals $1$ for the $10$ highest-accuracy workers and $0$ otherwise.

The base matrix is fixed within a seed, while the reward uses a fresh $V^{(t)}$ each round. In the \emph{stochastic} environment,
$V^{(t)}[n,i]\sim\mathrm{Beta}(\kappa V_{\mathrm{base}}[n,i],\kappa(1-V_{\mathrm{base}}[n,i]))$ independently across rounds with $\kappa=20$. In the \emph{regime-switching} environment, every $T/10$ rounds we sample a new set of $5$ workers and multiply their mean accuracies by $0.2$ for that block.

\subsection{Constraint configurations and optimization details}
Table~\ref{tab:sweep-configs} varies the expert-cap bound and the margin added to the baseline workload budget $M\cdot\overline{C_0}$. The first three configurations are used in the regret-convergence experiment, and all four are used in the policy comparison.

\begin{table}[h]
\centering
\caption{Constraint configurations used in the experiments.}
\label{tab:sweep-configs}
\begin{tabular}{lcc}
\toprule
Config & Budget margin & Expert cap $\alpha_{\mathrm{expert}}$ \\
\midrule
tight  & $+0.0$ & $2.0$ \\
medium & $+0.1$ & $2.5$ \\
loose  & $+0.5$ & $5.0$ \\
none   & $-$    & $\infty$ \\
\bottomrule
\end{tabular}
\end{table}

For the regret-convergence experiment, we set $K=50$, $M=5$, $T=10{,}000$, $N\in\{3,5\}$, and average over $5$ seeds. OMD-Approachability holds $q$ fixed for batches of $B=10$ rounds. At each batch endpoint $t=nB$, the empirical objective quantity is
\[
  R_1(n)=\frac1t\sum_{s=1}^{t}\Bigl(\mathbb E_{A\sim q_T}[r(A,V^{(s)})]-r_s\Bigr),
\]
where the expectation is estimated by $200$ Monte Carlo draws for each stored $V^{(s)}$. The realized constraint quantity is
\[
  R_2(n)=\mathrm{dist}\!\left(\frac1t\sum_{s=1}^{t}\ell_s,S^{\mathrm{ori}}\right).
\]
We use $\eta=\eta_{\mathrm{mult}}\sqrt{2TKM(1/\gamma+M-1)}$ with $\eta_{\mathrm{mult}}=0.03$, dual OGD step size $\eta_{\mathrm{ogd}}=0.05$, and exploration floor $\gamma=1/(10K)=0.002$.

For the policy comparison, we set $N=3$ and use $20$ independent seeds. OMD-Approachability again updates $q$ every $10$ rounds and uses $\eta_{\mathrm{mult}}=0.05$, $\eta_{\mathrm{ogd}}=0.05$, and $\gamma=0.002$.

\subsection{Baselines and implementation details}
All UCB- and TS-based baselines use a $10$-round round-robin warm-up. Since the simulation exposes only the aggregate team reward, these baselines assign the full observed reward $r_t$ to every selected arm as a proxy for its contribution.
\begin{itemize}
  \item \textbf{Random} selects $M$ workers uniformly without replacement and does not use the constraints.
  \item \textbf{Greedy} selects the $M$ workers with the largest empirical shared-reward means after warm-up.
  \item \textbf{CUCB} follows the combinatorial UCB principle of~\citet{chen2013combinatorial}, selecting the top-$M$ arms according to $\hat\mu_i+2\sqrt{2\log t/n_i}$.
  \item \textbf{CBwK} augments CUCB with a nonnegative Lagrangian multiplier for the operational constraints, following the constrained combinatorial-bandit perspective of~\citet{sankararaman2018combinatorial,liu2022combinatorial}. The dual step size is $10/\sqrt T$.
  \item \textbf{TS} uses independent Beta posteriors and selects the top-$M$ sampled values, following combinatorial Thompson sampling~\citep{wang2018thompson}. Selected arms receive updates $\alpha_i\mathrel{+}=r_t/N$ and $\beta_i\mathrel{+}=1-r_t/N$.
  \item \textbf{TS-BwK} adds the same Lagrangian constraint adjustment to TS; see also constrained Thompson sampling~\citep{deb2025thompson}. The dual step size is $10/\sqrt T$.
  \item \textbf{MoE} adapts sparse Mixture-of-Experts gating~\citep{shazeer2017outrageously}. A logit vector parameterizes the gate, Gumbel-Top-$M$ produces the selected workers, and REINFORCE updates the logits from the Lagrangian-adjusted aggregate reward. We use learning rate $0.01$ and dual step size $4\times10^{-4}$.
  \item \textbf{OMD-Approachability} is Algorithm~\ref{alg:approachability-OMD}; its routing distribution is updated by the constrained QP every $10$ rounds using the importance-weighted aggregate-reward gradient, and the dual iterate is updated by OGD on the unit ball.
\end{itemize}

\section{Auxiliary Lemmas}
\label{app:auxiliary}
This section collects the auxiliary results used in the proofs. 

\subsection{Proof of Lemma~\ref{lem:dist-transfer}}
\label{app:dist-transfer}
For any sequence $q_1,\ldots,q_T\in\Delta_K$, define
\[
  \phi:=\tfrac1T\sum_{t=1}^T\mathbb E[\ell_t],\qquad
  \varphi:=\tfrac1T\sum_{t=1}^T\sum_{i=1}^Kq_{t,i}c_i.
\]
Because both target sets are axis-aligned, it suffices to work coordinate-wise. For an upper-bound coordinate, averaging~\eqref{eq:sandwich} gives $0\le\phi^{(j)}\le M\varphi^{(j)}$, and therefore
\[
\mathrm{dist}\bigl(\phi^{(j)},(-\infty,u_j]\bigr)
=\max(\phi^{(j)}-u_j,0)
\le M\max(\varphi^{(j)}-u_j/M,0)
=M\,\mathrm{dist}\bigl(\varphi^{(j)},[0,u_j/M]\bigr).
\]
For a lower-bound coordinate, $\phi^{(j)}\ge\varphi^{(j)}$ and $\varphi^{(j)}\le\bar c$, so
\[
\mathrm{dist}\bigl(\phi^{(j)},[l_j,\infty)\bigr)
=\max(l_j-\phi^{(j)},0)
\le\max(l_j-\varphi^{(j)},0)
=\mathrm{dist}\bigl(\varphi^{(j)},[l_j,\bar c]\bigr).
\]
Combining the coordinate-wise inequalities and taking Euclidean norms gives
$\mathrm{dist}(\phi,S^{\mathrm{ori}})\le M\,\mathrm{dist}(\varphi,S)$.

\subsection{Other Lemmas}
Let $w_t$ denote the dual iterate in Algorithm~\ref{alg:approachability-OMD}. We use the following feasible sets:
\[
\mathcal{F}_{\circ}:=\{q\in\Delta_K: q_i\ge\gamma;\; \textstyle\sum_i q_i c_i \in S\}
\]
\[
\mathcal{F}_{\tau}:=\{q\in\Delta_K: q_i\ge\gamma;\; w_t^\top\textstyle\sum_i q_i c_i -h_S(w_t)\le 0,\;\forall t\in [\tau]\},
\]
The per-context sets $\mathcal{F}_{\circ}^{(x)}$ and $\mathcal{F}_{\tau}^{(x)}$ are defined analogously, with the condition over $t\in[\tau]$ restricted to rounds satisfying $x_t=x$. The sets $\mathcal{F}_{\tau}$ and $\mathcal{F}_{\tau}^{(x)}$ are random because they depend on the algorithmic iterates $w_t$.

For the conditional-expectation arguments, define the history available immediately before sampling $A_t$ by
\[
\mathcal H_{t-1}
=
\sigma\!\left(
w_1,\ldots,w_t,\;
q_1,\ldots,q_t,\;
A_1,r_1(A_1),\ldots,A_{t-1},r_{t-1}(A_{t-1})
\right).
\]
Under Algorithm~\ref{alg:approachability-OMD}, $q_t$ and $w_t$ are $\mathcal H_{t-1}$-measurable, while $g_t$ is $\mathcal H_t$-measurable.

\begin{lemma}[Three-point identity]\label{lem:three-point}
For any $q, q_t, q_{t+1} \in \mathbb{R}^K$,
\[
\langle q_{t+1} - q_t,\, q - q_{t+1} \rangle
= \tfrac{1}{2}\|q - q_t\|^2 - \tfrac{1}{2}\|q - q_{t+1}\|^2 - \tfrac{1}{2}\|q_{t+1} - q_t\|^2.
\]
\end{lemma}

\begin{lemma}\label{lem:distance2support} Let $S$ be a convex and compact set, 
$\mathrm{dist}(u, S) = \max_{\|w\| \leq 1} \left\{ w^\top u - h_S(w) \right\}.$
\end{lemma}
\begin{proof}
The result follows from the support-function representation in Lemma~13.5 of~\citep{Hazan}. Specifically,
\begin{align*}
\max_{\|w\| \leq 1} \left\{ w^\top u - h_S(w) \right\}
&= \max_{\|w\| \leq 1} \left\{ w^\top u - \max_{\mathbf{x} \in S} w^\top \mathbf{x} \right\} \\
&= \max_{\|w\| \leq 1} \min_{\mathbf{x} \in S} \left\{ w^\top u - w^\top \mathbf{x} \right\}\\
&= \min_{\mathbf{x} \in S} \max_{\|w\| \leq 1} \left\{ w^\top u - w^\top \mathbf{x} \right\}
\quad \text{(minimax theorem)} \\
&= \min_{\mathbf{x} \in S} \| \mathbf{x} - u \| \\
&= \mathrm{dist}(u, S).
\end{align*}
\end{proof}

\section{Proof of Theorem~\ref{thm:reg1}}
\label{app:thm1}


\begin{proof}[Proof of Theorem~\ref{thm:reg1}]
Let $q^\star\in\arg\max_{q\in\mathcal F_{\circ}}\sum_{t=1}^T\Phi^{r_t}(q)$. By the definition of $\mathrm{OPT}_{\mathrm{sur}}$ in~\eqref{equ:opt-define} and Assumption~\ref{ass:non-zero},
\begin{equation}
\sum_{t=1}^T \Phi^{r_t}(q^\star)\ge T\cdot \mathrm{OPT}_{\mathrm{sur}}.
\label{eq:opt-lower}
\end{equation}
Therefore
\begin{equation}
T\cdot\mathrm{OPT}_{\mathrm{sur}}-\sum_{t=1}^T \Phi^{r_t}(q_t)
\le
\sum_{t=1}^T\bigl(\Phi^{r_t}(q^\star)-\Phi^{r_t}(q_t)\bigr).
\label{equ:relaxation-SGD}
\end{equation}

\textbf{Step 1: one-step mirror-descent inequality.}
Lemma~\ref{lem:distance2support} and $\|w_t\|\le1$ imply $\mathcal F_{\circ}\subseteq\mathcal F_t$ for every realization. Hence $q^\star\in\mathcal F_t$ almost surely. Because $q_{t+1}$ solves the proximal problem~\eqref{equ:online-opt}, its first-order optimality condition gives
\[
\langle g_t-\eta(q_{t+1}-q_t),\,q-q_{t+1}\rangle\le 0,
\qquad \forall q\in\mathcal F_t.
\]
Set $q=q^\star$ and apply the three-point identity in Lemma~\ref{lem:three-point}:
\[
\langle g_t,q^\star-q_t\rangle
\le
\tfrac{\eta}{2}
\left(
\|q^\star-q_t\|^2
-
\|q^\star-q_{t+1}\|^2
-
\|q_{t+1}-q_t\|^2
\right)
+
\langle g_t,q_{t+1}-q_t\rangle.
\]

Concavity of $\Phi^{r_t}$ gives
\[
\Phi^{r_t}(q^\star)-\Phi^{r_t}(q_t)
\le
\langle\nabla\Phi^{r_t}(q_t),\,q^\star-q_t\rangle.
\]
Insert the stochastic gradient estimator $g_t$ to obtain
\[
\begin{aligned}
\Phi^{r_t}(q^\star)-\Phi^{r_t}(q_t)
&\le
\langle g_t,q^\star-q_t\rangle
+
\langle\nabla\Phi^{r_t}(q_t)-g_t,\,
q^\star-q_t\rangle
\\
&\le
\tfrac{\eta}{2}
\left(
\|q^\star-q_t\|^2
-
\|q^\star-q_{t+1}\|^2
-
\|q_{t+1}-q_t\|^2
\right)
\\
&\qquad
+
\langle g_t,q_{t+1}-q_t\rangle
+
\langle\nabla\Phi^{r_t}(q_t)-g_t,\,
q^\star-q_t\rangle.
\end{aligned}
\]

Summing over $t=1,\ldots,T$ and telescoping the squared-distance terms yields
\begin{equation}
\begin{aligned}
\sum_{t=1}^T
\bigl(
\Phi^{r_t}(q^\star)-\Phi^{r_t}(q_t)
\bigr)
&\le
\tfrac{\eta}{2}\|q^\star-q_1\|^2
+
\sum_{t=1}^T
\langle g_t,q_{t+1}-q_t\rangle
\\
&\qquad
-
\tfrac{\eta}{2}
\sum_{t=1}^T
\|q_{t+1}-q_t\|^2
\\
&\qquad
+
\sum_{t=1}^T
\langle
\nabla\Phi^{r_t}(q_t)-g_t,\,
q^\star-q_t
\rangle.
\end{aligned}
\label{equ:decomposition-SGD}
\end{equation}

\textbf{Step 2: stochastic-gradient error.}
For the final term in~\eqref{equ:decomposition-SGD}, Lemma~6.6 of~\cite{pasteris2023sum} gives
\[
\mathbb E[g_t\mid\mathcal H_{t-1}]
=
\nabla\Phi^{r_t}(q_t).
\]
Since both $q^\star$ and $q_t$ are $\mathcal H_{t-1}$-measurable,
\[
\mathbb E\!\left[
\langle
\nabla\Phi^{r_t}(q_t)-g_t,\,
q^\star-q_t
\rangle
\mid
\mathcal H_{t-1}
\right]
=0.
\]
Taking expectations and applying the tower property therefore gives
\[
\mathbb E\!\left[
\sum_{t=1}^T
\langle
\nabla\Phi^{r_t}(q_t)-g_t,\,
q^\star-q_t
\rangle
\right]
=0.
\]

\textbf{Step 3: second moment of the bandit gradient.}
Young's inequality bounds the remaining inner-product term in~\eqref{equ:decomposition-SGD} as
\[
\langle g_t,q_{t+1}-q_t\rangle
\le
\tfrac{2}{\eta}\|g_t\|^2
+
\tfrac{\eta}{2}\|q_{t+1}-q_t\|^2.
\]
After summing over $t$, the second term on the right cancels the negative squared-step term in~\eqref{equ:decomposition-SGD}. Conditioning on $\mathcal H_{t-1}$, Lemma~6.14 of~\cite{pasteris2023sum} gives
\[
\mathbb E[g_{t,i}^2\mid\mathcal H_{t-1}]
\le
\tfrac{M}{q_{t,i}}
+
M(M-1).
\]
Since $q_{t,i}\ge\gamma$ almost surely,
\[
\mathbb E[g_{t,i}^2\mid\mathcal H_{t-1}]
\le
\tfrac{M}{\gamma}
+
M(M-1).
\]
Taking expectations and summing over $i$,
\[
\mathbb E[\|g_t\|^2]
=
\sum_{i=1}^K
\mathbb E[g_{t,i}^2]
\le
KM\left(\frac1\gamma+M-1\right).
\]
Summing this bound over the horizon gives
\[
\sum_{t=1}^T
\mathbb E[\|g_t\|^2]
\le
TKM\left(\frac1\gamma+M-1\right).
\]

\textbf{Step 4: aggregate bound.}
Take expectations in~\eqref{equ:relaxation-SGD} and combine the preceding bounds with~\eqref{equ:decomposition-SGD}. Since the squared Euclidean diameter of the simplex is at most $2$, $\|q^\star-q_1\|^2\le2$, and we obtain
\[
\begin{aligned}
\mathbb E\!\left[
T\cdot\mathrm{OPT}_{\mathrm{sur}}
-
\sum_{t=1}^T
\Phi^{r_t}(q_t)
\right]
&\le
\eta
+
\tfrac{2}{\eta}
\sum_{t=1}^T
\mathbb E[\|g_t\|^2]
\\
&\le
\eta
+
\tfrac{2}{\eta}
TKM
\left(
\frac1\gamma+M-1
\right).
\end{aligned}
\]

\textbf{Step 5: choice of proximal parameter.}
Minimizing the upper bound over $\eta$ gives
\[
\eta^\star
=
\sqrt{
2TKM
\left(
\frac1\gamma+M-1
\right)
}.
\]
Substituting $\eta^\star$ yields
\[
\mathbb E\!\left[
T\cdot\mathrm{OPT}_{\mathrm{sur}}
-
\sum_{t=1}^T
\Phi^{r_t}(q_t)
\right]
\le
2
\sqrt{
2TKM
\left(
\frac1\gamma+M-1
\right)
},
\]
which proves the theorem.
\end{proof}

\section{Proof of Theorem~\ref{thm:reg2}}
\label{app:thm2}
We decompose the constraint regret into a sampling term and an approachability term:
\begin{equation}
\begin{aligned}
  \mathcal{R}_2(T) 
  &=\operatorname{dist}\!\left(
    \tfrac{1}{T}\sum_{t=1}^T \ell_t,\; S^{\mathrm{ori}}
  \right)=
  \operatorname{dist}\!\left(
    \tfrac{1}{T}\sum_{t=1}^T \sum_{i=1}^K \mathbf{1}[i \in A_t]\, c_i,\; S^{\mathrm{ori}}
  \right)\\
  &\le \operatorname{dist}\!\left(
    \tfrac{1}{T}\sum_{t=1}^T \sum_{i=1}^K \mathbf{1}[i \in A_t]\, c_i,\; \tfrac{1}{T}\sum_{t=1}^T \E[\ell_t]
  \right) + \operatorname{dist}\!\left(
    \tfrac{1}{T}\sum_{t=1}^T \E[\ell_t],\; S^{\mathrm{ori}}
  \right)\\
  &\le\underbrace{
\dist\!\left(
\tfrac{1}{T}\sum_{t=1}^T \sum_{i=1}^K \mathbf{1}[i \in A_t]\, c_i,\;
\tfrac{1}{T}\sum_{t=1}^T \E[\ell_t]
\right)
}_{\text{sampling error}}
+
M\cdot \underbrace{
\dist\!\left(
\tfrac{1}{T}\sum_{t=1}^T\sum_{i=1}^K q_{t,i}\,c_i,\;
S
\right)
}_{\text{approachability error}},
\end{aligned}
\label{equ:R2-decompose}
\end{equation}
The first inequality is the triangle inequality, and the second applies Lemma~\ref{lem:dist-transfer}. We bound the two terms separately.

\begin{lemma}[Approachability error]\label{lem:approach-error}
Let $q_t$, $t\in[T]$, be generated by Algorithm~\ref{alg:approachability-OMD} with OCO subroutine $\mathcal A$. Then
\[
\mathrm{dist}\!\left(\tfrac{1}{T}\sum_{t=1}^T\sum_{i=1}^K q_{t,i}\,c_i,\,S\right)\;\le\;\tfrac{\mathrm{Regret}_T(\mathcal{A})}{T}.
\]
\end{lemma}
\begin{proof}
The argument follows Theorem~13.7 of~\citep{Hazan}. Set $u_t:=\sum_{i=1}^Kq_{t,i}c_i$. Since $q_t$ satisfies the constraint in~\eqref{equ:online-opt}, the dual iterate $w_t$ obeys
\begin{equation}
f_{t}(w_t)= w_t^\top u_{t} - h_S(w_t) \leq 0.
\label{equ:nonpositive-f}
\end{equation}
Let $\bar u_T:=\tfrac1T\sum_{t=1}^Tu_t$. Lemma~\ref{lem:distance2support} then yields
\begin{align*}
\mathrm{dist}(\bar{u}_T, S)
&= \max_{\|w\| \leq 1} \left\{ w^\top \bar{u}_T - h_S(w) \right\} \\
&= \max_{w^* \in \mathbb{B}} \tfrac{1}{T} \sum_{t=1}^{T} f_t(w^*)
\quad \text{(definition of } f_t\text{)} \\
&\leq \tfrac{1}{T} \sum_{t=1}^{T} f_t(w_t) + \tfrac{\mathrm{Regret}_T(\mathcal{A})}{T}
\quad \text{(OCO guarantee of } \mathcal{A}\text{)} \\
&\leq \tfrac{\mathrm{Regret}_T(\mathcal{A})}{T}
\quad \text{(Equation~\eqref{equ:nonpositive-f})}.
\end{align*}
\end{proof}

\begin{lemma}[Sampling error]\label{lem:sampling-error}
With probability at least $1-\delta$, for each coordinate $j\in[d]$,
\[
\left|\tfrac{1}{T}\sum_{t=1}^T\sum_{i=1}^K
(\mathbf{1}[i\in A_t]-\mathbb{E}[\mathbf{1}[i\in A_t]])c_i^{(j)}\right|\le   {M}\bar{c}\sqrt{\tfrac{2\log(2d/\delta)}{T}}.
\]
Therefore, for an absolute constant $L>0$,
\[
\E\;\dist\!\left(
\tfrac{1}{T}\sum_{t=1}^T \sum_{i=1}^K \mathbf{1}[i \in A_t]\, c_i,\;
\tfrac{1}{T}\sum_{t=1}^T \E [\ell_t]
\right)
\le
  {M}\bar{c}L\sqrt{\tfrac{d\log(2d)}{T}}.
\]
\end{lemma}
\begin{proof}
\textbf{Step 1: high-probability bound.}
Define $Y_t:=\sum_{i=1}^K(\mathbf{1}[i\in A_t]-\mathbb{E}[\mathbf{1}[i\in A_t]\mid\mathcal H_{t-1}])c_i$. For each $j\in[d]$, $\mathbb{E}[Y_t^{(j)}\mid\mathcal H_{t-1}]=0$, so $(Y_t^{(j)},\mathcal H_t)$ is a martingale difference sequence. Moreover, $Y_t^{(j)}\in[-M\bar c,M\bar c]$ because $c_i^{(j)}\in[0,\bar c]$ and $|A_t|\le M$. Let $\bar Y:=\tfrac1T\sum_{t=1}^TY_t$. Azuma--Hoeffding gives
\[
\Pr\left(\left|\bar{Y}^{(j)}\right|\geq \epsilon\right)
\leq 2\exp\left(-\tfrac{T\epsilon^2}{2  {M^2}\bar{c}^2}\right).
\]
Applying a union bound over $j\in[d]$ and setting $\epsilon={M}\bar c\sqrt{2\log(2d/\delta)/T}$ gives, with probability at least $1-\delta$,
\[
\left|\bar{Y}^{(j)}\right|
\leq   {M}\bar{c}\sqrt{\tfrac{2\log(2d/\delta)}{T}}
\quad \text{for all } j\in[d].
\]
Denote this event by $\mathcal E(\delta)$.

\textbf{Step 2: Euclidean norm.}
On $\mathcal E(\delta)$,
\[
\|\bar{Y}\|_2 = \sqrt{\sum_{j=1}^d (\bar{Y}^{(j)})^2} \leq \sqrt{d \cdot \tfrac{2  {M^2}\bar{c}^2\log(2d/\delta)}{T}} = \sqrt{\tfrac{2d  {M^2}\bar{c}^2\log(2d/\delta)}{T}}.
\]

\textbf{Step 3: expectation.}
\[
\mathbb{E}\|\bar{Y}\|_2 = \int_0^\infty \mathbb{P}\!\left(\|\bar{Y}\|_2 > u\right)du.
\]
Choose $\delta=\delta(u)$ so that the high-probability threshold equals $u$:
\[
\sqrt{\tfrac{2d  {M^2}\bar{c}^2\log(2d/\delta)}{T}} = u
\implies
\delta(u) = 2d\exp\!\left(-\tfrac{Tu^2}{2d  {M^2}\bar{c}^2}\right).
\]
On $\mathcal E(\delta)$, $\|\bar Y\|_2\le u$, hence $\mathbb P(\|\bar Y\|_2>u)\le\delta(u)$. Therefore
\[
\mathbb{E}\|\bar{Y}\|_2 \leq \int_0^\infty \min\!\left(1,\ 2d\exp\!\left(-\tfrac{Tu^2}{2d  {M^2}\bar{c}^2}\right)\right)du.
\]
Let $u_0=\sqrt{2d{M^2}\bar c^2\log(2d)/T}$, for which $2d\exp(-Tu_0^2/(2d{M^2}\bar c^2))=1$. Splitting the integral at $u_0$ gives
\[
\mathbb{E}\|\bar{Y}\|_2
\;\leq\;
\underbrace{\int_0^{u_0} 1\, du}_{=\;u_0}
\;+\;
\int_{u_0}^\infty 2d\exp\!\left(-\tfrac{Tu^2}{2d  {M^2}\bar{c}^2}\right)du.
\]
Apply the Gaussian tail bound $\int_\nu^\infty e^{-t^2/a^2}dt\le\tfrac{a^2}{2\nu}e^{-\nu^2/a^2}$ with $a^2=2d{M^2}\bar c^2/T$:
\[
\int_{u_0}^\infty 2d\, e^{-Tu^2/(2d  {M^2}\bar{c}^2)}\,du
\le \tfrac{2d^2  {M^2}\bar{c}^2}{Tu_0} \cdot e^{-Tu_0^2/(2d  {M^2}\bar{c}^2)}
= \tfrac{2d^2  {M^2}\bar{c}^2}{Tu_0} \cdot \tfrac{1}{2d}
= \sqrt{\tfrac{d  {M^2}\bar{c}^2}{2T\log(2d)}}.
\]
Combining the two terms,
\[
\mathbb{E}\|\bar{Y}\|_2
\;\leq\;
\sqrt{\tfrac{2d  {M^2}\bar{c}^2\log(2d)}{T}} + \sqrt{\tfrac{d  {M^2}\bar{c}^2}{2T\log(2d)}}
\;\leq\;
  {M}\bar{c}L\sqrt{\tfrac{d\log(2d)}{T}}
\]
for an absolute constant $L>0$.
\end{proof}

\begin{proof}[Proof of Theorem~\ref{thm:reg2}]
Taking expectations in~\eqref{equ:R2-decompose} and applying Lemmas~\ref{lem:approach-error} and~\ref{lem:sampling-error} proves the theorem.
\end{proof}

\section{Contextual Algorithm and Proof of Corollary~\ref{cor:contextual-reg1}}
\label{app:contextual}

\begin{algorithm}[H]\label{alg:contextual-OMD}
\caption{Contextual OMD with Approachability Constraints}
\begin{enumerate}
\item \textbf{Input:} rescaled set $S$, OCO algorithm $\mathcal{A}$, context set $\mathcal{X}$, step-size constant $a=2\sqrt{KM(\tfrac1\gamma+M-1)}$
\item Set $\mathbb{B}\subset\mathbb{R}^d$ to the unit Euclidean ball, used as the decision set for $\mathcal{A}$
\item For each $x\in\mathcal{X}$: set $q^{(x)}_i:=1/K$ for $i\in[K]$, $g^{(x)}:=\mathbf{0}$, and visit count $n^{(x)}:=0$
\item For $t = 1, \ldots, T$:
\begin{enumerate}
    \item Observe the context $x \leftarrow x_t$ and increment $n^{(x)} \leftarrow n^{(x)}+1$
    \item Set $f_{t-1}(w) = w^\top u_{t-1} - h_S(w)$, where $u_{t-1} = \sum_{i=1}^K q_{t-1,i}^{(x)}\; c_i$
    \item Query $\mathcal{A}$: $w_t \leftarrow \mathcal{A}(f_1, \ldots, f_{t-1})$
    \item Set the per-context anytime proximal parameter $\eta\leftarrow a\sqrt{n^{(x)}}$
    \item Update only the current context's iterate, using its own stored gradient $g^{(x)}$:
    \begin{equation}
    \label{equ:online-opt-ctx}
    \begin{split}
        q^{(x)} \leftarrow \arg\max_{q\in\Delta_K:\; q\ge\gamma} & \quad q^\top g^{(x)} -\tfrac{\eta}{2}\bigl\|q-q^{(x)}\bigr\|^2 \\
        \text{s.t. }& w_t^\top  \sum_{i=1}^K q_i\; c_i - h_S(w_t) \leq 0
    \end{split}
    \end{equation}
    \item Draw $a_{t,1},\dots,a_{t,M}$ independently from $q^{(x)}$ and set $A_t \leftarrow \{a_{t,j} : j \in [M]\}$
    \item Observe $r_t(A_t)$
    \item Refresh the current context's gradient: for each $i \in [K]$,
    $g^{(x)}_i \leftarrow \tfrac{r_t(A_t)}{q_{i}^{(x)}} \sum_{j=1}^M \ind{a_{t,j}=i}$
\end{enumerate}
\end{enumerate}
\end{algorithm}

\begin{proof}[Proof of Corollary~\ref{cor:contextual-reg1}]
Let $G^2:=KM(1/\gamma+M-1)$, so $a=2G$. Fix $x\in\mathcal X$, and let $T^{(x)}$ be the number of rounds for which $x_t=x$. Index these visits by $s=1,\ldots,T^{(x)}$. Restricted to this subsequence, the iterates $\{q_s^{(x)}\}$ follow the same primal update as Algorithm~\ref{alg:approachability-OMD}, with proximal parameter $\eta_s^{(x)}=a\sqrt s$.

Apply the one-step inequality from Step~1 of the proof of Theorem~\ref{thm:reg1} to the subsequence for context $x$, using the corresponding filtration and conditional expectation. Summing over $s=1,\ldots,T^{(x)}$ gives, for any comparator $q_\star^{(x)}\in\mathcal F_\circ^{(x)}$,
\[
\mathbb E\!\left[
\sum_{s=1}^{T^{(x)}}
\bigl(
\Phi^{r_s}(q^{(x)}_{\star})-\Phi^{r_s}(q_s^{(x)})
\bigr)
\right]
\le
\sum_{s=1}^{T^{(x)}}
\tfrac{\eta_s^{(x)}}2
\Bigl(
\|q^{(x)}_{\star}-q_s^{(x)}\|^2
-
\|q^{(x)}_{\star}-q_{s+1}^{(x)}\|^2
\Bigr)
+
\sum_{s=1}^{T^{(x)}}
\tfrac{2G^2}{\eta_s^{(x)}}.
\]
The stochastic inner-product terms have zero expectation by the same argument used in Theorem~\ref{thm:reg1}.

For the first sum, Abel summation and $\|q^{(x)}_\star-q_s^{(x)}\|^2\le2$ give
\[
\sum_{s=1}^{T^{(x)}}\tfrac{\eta^{(x)}_s}{2}\Bigl(\bigl\|q^{(x)}_{\star}-q^{(x)}_s\bigr\|^2
        -\bigl\|q^{(x)}_{\star}-q^{(x)}_{s+1}\bigr\|^2\Bigr)
  \;\le\;
  \eta^{(x)}_1 + \sum_{s=2}^{T^{(x)}}\bigl(\eta^{(x)}_s-\eta^{(x)}_{s-1}\bigr)
  \;=\; \eta^{(x)}_{T^{(x)}} \;=\; a\sqrt{T^{(x)}}.
\]
For the second sum, we use $\sum_{s=1}^n s^{-1/2}\le2\sqrt n$:
\[
\sum_{s=1}^{T^{(x)}}\frac1{\eta_s^{(x)}}
=
\frac1a\sum_{s=1}^{T^{(x)}}s^{-1/2}
\le
\tfrac{2}{a}\sqrt{T^{(x)}}.
\]
Thus the expected regret for context $x$ satisfies
\[
\mathbb E\!\left[
\sum_{s=1}^{T^{(x)}}
\bigl(
\Phi^{r_s}(q^{(x)}_{\star})-\Phi^{r_s}(q_s^{(x)})
\bigr)
\right]
\le
\left(a+\tfrac{4G^2}{a}\right)\sqrt{T^{(x)}}
=
4G\sqrt{T^{(x)}},
\]
where the last equality uses $a=2G$.

Finally, summing over contexts and applying Cauchy--Schwarz, $\sum_{x\in\mathcal X}\sqrt{T^{(x)}}\le\sqrt{|\mathcal X|T}$, gives
\[
\begin{aligned}
T\,\mathbb E\!\left[R_1^{\mathrm{ctx}}(T)\right]
&\le
4G
\sum_{x\in\mathcal{X}}\sqrt{T^{(x)}}
\\
&\le
4G\sqrt{|\mathcal{X}|\,T}
\\
&=
4
\sqrt{
|\mathcal{X}|\,TKM
\left(
\frac1\gamma+M-1
\right)
}.
\end{aligned}
\]
This proves the corollary.
\end{proof}

\section{Suboptimality of the Surrogate-Restricted Benchmark: Proofs and Additional Bounds}
\label{app:subopt}

For $q\in\Delta_K$, recall that $B(q)$ denotes the random subset obtained by drawing
$M$ actions independently from $q$, and write
\begin{equation}
 g_M(p):=1-(1-p)^M.
\label{eq:subopt-definitions}
\end{equation}
Thus $g_M(q_i)=\Pr(i\in B(q))$. All results in this appendix concern the
benchmarks in~\eqref{eq:subopt-gap}. Since the surrogate feasible set is nonempty,
both benchmark optima exist: the relevant feasible sets are closed subsets of the
compact simplex, and their objectives are continuous. The case $M=1$ is exact,
because the surrogate and original expected-cost constraints coincide.

\subsection{Proof of Proposition~\ref{prop:subopt-margin}}
\label{app:subopt-margin-proof}

\begin{proof}
Write $\delta=\delta(q^{\mathrm{ori}})$,
$\theta=\delta/(\sigma+\delta)$, and define
\[
 \widetilde q=(1-\theta)q^{\mathrm{ori}}+\theta q^{\mathrm{ref}}.
\]
For an upper constraint, affinity of the surrogate cost in $q$ implies
\[
 M\bigl(\sum_i \widetilde q_i c_i\bigr)^{(j)}-u_j
 =(1-\theta)\bigl(M\bigl(\sum_i q_i^{\mathrm{ori}}c_i\bigr)^{(j)}-u_j\bigr)
   +\theta\bigl(M\bigl(\sum_i q_i^{\mathrm{ref}}c_i\bigr)^{(j)}-u_j\bigr)
 \le(1-\theta)\delta-\theta\sigma=0.
\]
For a lower constraint, the same calculation gives
\[
 l_j-\bigl(\sum_i \widetilde q_i c_i\bigr)^{(j)}
 \le(1-\theta)\delta-\theta\sigma=0.
\]
Thus $\sum_i\widetilde q_i c_i\in S$. By Lemma~1,
$\frac1T\sum_{t=1}^T\Phi^{r_t}$ is concave, and therefore
\[
 \mathrm{OPT}_{\mathrm{sur}}\ge\frac1T\sum_{t=1}^T\Phi^{r_t}(\widetilde q)
 \ge(1-\theta)\mathrm{OPT}_{\mathrm{ori}}
       +\theta\frac1T\sum_{t=1}^T\Phi^{r_t}(q^{\mathrm{ref}}).
\]
Rearranging proves the first upper bound in
\eqref{eq:subopt-margin-bound}. Because $q^{\mathrm{ref}}$ is also
original-feasible and rewards lie in $[0,1]$,
$0\le\mathrm{OPT}_{\mathrm{ori}}-\frac1T\sum_{t=1}^T\Phi^{r_t}(q^{\mathrm{ref}})\le1$,
which proves the second bound. If $\delta=0$, then $\theta=0$ and the same
argument gives $\mathrm{OPT}_{\mathrm{ori}}-\mathrm{OPT}_{\mathrm{sur}}=0$.
\end{proof}

\subsection{Proof and sharpness of Corollary~\ref{cor:subopt-null}}
\label{app:subopt-null-proof}

\begin{proof}

\textbf{Step 1.}
Let $q^{\mathrm{ori}}$ be an original optimizer and let
$p_i:=g_M(q_i^{\mathrm{ori}})=\Pr(i\in B(q^{\mathrm{ori}}))$ be its activation marginals, i.e., the probability that action \(i\) appears at least once among the \(M\) iid draws.
Let $I_i:=\mathbf{1}\{i\in B(q^{\mathrm{ori}})\}$.
Since at most
$M$ distinct actions are selected, we have:
$$\sum_i p_i=\sum_i \mathbb{E}[I_i]=\mathbb{E}\left[\sum_i I_i\right]=\mathbb E[|B(q^{\mathrm{ori}})|]\le M.$$

Let $i_0$ be the dull action with $c_{i_0}=\mathbf{0}$ and
$r_t(A\cup\{i_0\})=r_t(A)$ for every $A\subseteq[K]$.
Let $\textbf{e}_{i_0}$ denote the point mass at $i_0$, and construct a new probability distribution
\begin{equation}
 \widetilde q:=\frac{p}{M}
       +\left(1-\frac1M\sum_i p_i\right)\textbf{e}_{i_0}.
 \label{eq:subopt-null-construction}
\end{equation}
Because the dull action induces zero cost $c_{i_0}=\mathbf{0}$,
\begin{equation}
 M\sum_i \widetilde q_i c_i=\sum_i p_i c_i=\sum_i g_M(q_i^{\mathrm{ori}})c_i.
\label{eq:subopt-null-cost}
\end{equation}
Since $q^{\mathrm{ori}}$ is the optimizer of $\mathrm{OPT}_{\mathrm{ori}}$, we have
$\bigl(\sum_i g_M(q_i^{\mathrm{ori}})c_i\bigr)^{(j)}\le u_j$ for all upper original constraints $j\in\mathcal{U}$.
So $\widetilde q$ satisfies every upper surrogate constraint.

\textbf{Step 2.}
We next prove the reward comparison. For simplicity, we first assume that the sum-max representation is nonnegative, i.e., \(V_{t,k,i}\ge0\) for all \(t,k,i\), and remove this assumption afterward.
Then, $r_t(\{i_0\})=0$ implies $V_{t,k,i_0}=0$ for all \(t,k\).

Fix $t,k$, and a threshold $z>0$, and define
$D_z:=\{i\in[K]:V_{t,k,i}\ge z\}$. This set excludes $i_0$.
Write $\rho_z:=\Pr(B(q^{\mathrm{ori}})\cap D_z\ne\varnothing)$.
The union bound gives $\rho_z\le\sum_{i\in D_z}p_i$.
Then, we have
\begin{equation}
\Pr(B(\widetilde q)\cap D_z\ne\varnothing)
=g_M\left(\frac{1}{M}\sum_{i\in D_z}p_i\right)
\ge g_M\left(\frac{\rho_z}{M}\right)
\ge \left(1-\left(1-\frac{1}{M}\right)^M\right)\rho_z.
\label{equ:coverage-lower-bound}
\end{equation}

The equality uses \eqref{eq:subopt-null-construction}, the definition of $\widetilde q$ and $i_0\notin D_z$.
The middle inequality uses $\rho_z\le\sum_{i\in D_z}p_i$ and that the map $x\mapsto g_M(x)$ is increasing on $[0,1]$.
The last inequality follows from the concavity of $x\mapsto g_M(x/M)$ on $[0,1]$ and the endpoint values $g_M(0)=0$ and $g_M(1/M)=1-(1-1/M)^M$, which imply
$$
g_M(x/M)\ge xg_M(1/M)+(1-x)g_M(0/M)
=\left(1-\left(1-\frac1M\right)^M\right)x.
$$


\textbf{Step 3.}
For \(X_q:=\max_{i\in B(q)}V_{t,k,i}\ge0\), the layer-cake representation and Tonelli's theorem give
$$
\mathbb{E}[X_q]
=
\int_0^\infty
\Pr\!\left(B(q)\cap D_z\neq\varnothing\right)\,dz.
$$
Therefore, by \eqref{equ:coverage-lower-bound},
\[
\mathbb{E}\!\left[
\max_{i\in B(\widetilde q)}V_{t,k,i}
\right]
\ge \left(1-\left(1-\frac{1}{M}\right)^M\right) \int_0^\infty\rho_z\,dz
=
\left(1-\left(1-\frac{1}{M}\right)^M\right)
\mathbb{E}\!\left[
\max_{i\in B(q^{\mathrm{ori}})}V_{t,k,i}
\right].
\]
Summing over \(k\) and averaging over \(t\) yields
$$
\frac1T\sum_{t=1}^T\Phi^{r_t}(\widetilde q)
\ge
\left(1-\left(1-\frac{1}{M}\right)^M\right)
\frac1T\sum_{t=1}^T\Phi^{r_t}(q^{\mathrm{ori}}).
$$
Since the same \(\widetilde q\) satisfies this inequality for every \(t\), this proves
\(\eqref{eq:subopt-null-bound}\).

\textbf{Step 4.}
The inequality
$(1-1/M)^M\le e^{-1}$ follows from $\log(1-x)\le-x$ for $0<x<1$;
the case $M=1$ is immediate.
\end{proof}

\paragraph{Lifting the nonnegative assumption.}
Even without assuming \(V_{t,k,i}\ge0\), the presence of the dull action \(i_0\) allows an equivalent nonnegative sum-max representation.
Since \(r_t(\{i_0\})=0\), we have \(\sum_kV_{t,k,i_0}=0\). Also,

$$
0=r_t(\{i,i_0\})-r_t(\{i\})
=\sum_k[V_{t,k,i_0}-V_{t,k,i}]_+,
$$

so \(V_{t,k,i}\ge V_{t,k,i_0}\). Hence

$$
W_{t,k,i}:=V_{t,k,i}-V_{t,k,i_0}\ge0,\qquad W_{t,k,i_0}=0,
$$

and \(W\) gives the same reward representation.

\paragraph{Sharpness.}
Take $K=2$, let action $2$ be the dull action, and use
$c_1=1$, $c_2=0$, $u=1$, and $r_t(A)=\boldsymbol1\{1\in A\}$.
The original optimum is $1$, attained by $q=(1,0)$.
Surrogate feasibility requires $q_1\le1/M$, so
$\mathrm{OPT}_{\mathrm{sur}}=g_M(1/M)$. Thus the factor cannot be improved
uniformly under the corollary's assumptions. For $M\ge2$, this example
also satisfies the existing Assumption~1 whenever $0<\gamma\le1/M$.

\end{document}